\documentclass[letterpaper]{article} 
\usepackage{aaai25}  
\usepackage{times}  
\usepackage{helvet}  
\usepackage{courier}  
\usepackage[hyphens]{url}  
\usepackage{graphicx} 
\usepackage{natbib}  
\usepackage{caption} 
\usepackage{algorithm}
\usepackage{algorithmicx}
\usepackage[noend]{algpseudocode}
\usepackage{makecell}

\usepackage{newfloat}
\usepackage{listings}
\DeclareCaptionStyle{ruled}{labelfont=normalfont,labelsep=colon,strut=off} 
\floatstyle{ruled}
\newfloat{listing}{tb}{lst}{}
\floatname{listing}{Listing}
\title{Information-Theoretic Generative Clustering of Documents}
\author{
Xin Du and Kumiko Tanaka-Ishii
}
\affiliations{
    Waseda University\\
    4-6-1 Okubo, Shinjuku-ku, Tokyo, 169-8555 Japan\\
    duxin@aoni.waseda.jp, kumiko@waseda.jp
}

\usepackage{amssymb}
\usepackage{amsthm}
\usepackage{amsmath}

\newtheorem{proposition}{Proposition}

\newcommand{\dvar}{X}

\newcommand{\qvar}{Y}

\newcommand{\doc}{x}

\newcommand{\query}{y}

\newcommand{\docs}{\mathrm{X}}

\newcommand{\queries}{\mathrm{Y}}

\usepackage{mathrsfs}

\newcommand{\qspace}{\mathcal{Y}}

\usepackage{tabularx}

\makeatletter
\newcommand{\multilinecol}[2]{%
\begin{tabularx}{\dimexpr\linewidth-\ALG@thistlm-0.05cm}[t]{@{}X@{}cp{1.0cm}@{}}
#1 && #2
\end{tabularx}
}
\makeatother

\usepackage{enumitem}
\usepackage{booktabs}
\usepackage{arydshln}
\usepackage{tabularx}
\usepackage{multicol}

\algdef{SE}[DOWHILE]{Do}{doWhile}{\algorithmicdo}[1]{\algorithmicwhile\ #1}%

\begin{document}

\maketitle

\begin{abstract}
We present {\em generative clustering} (GC) for clustering a set of documents,
$\docs$, by using texts $\queries$ generated by large language models (LLMs)
instead of by clustering the original documents $\docs$. Because LLMs provide
probability distributions, the similarity between two documents can be
rigorously defined in an information-theoretic manner by the KL divergence. We
also propose a natural, novel clustering algorithm by using importance sampling.
We show that GC achieves the state-of-the-art performance, outperforming any
previous clustering method often by a large margin. Furthermore, we show an
application to generative document retrieval in which documents are indexed via
hierarchical clustering and our method improves the retrieval accuracy.
\end{abstract}

%
\begin{links}
     \link{Code}{https://github.com/kduxin/lmgc}
\end{links}

\section{Introduction}

Document clustering has long been a foundational problem in data
science. Traditionally, every document $\doc\in\docs$ is first
translated into some computational representation
$\query\in\mathbb{R}^n$: in the 90s, as a bag-of-words vector, and
more recently, by using BERT \citep{devlin2019bert}. A clustering
algorithm, typically $k$-means, is then applied.

While large language models (LLMs) have been applied to diverse tasks,
their use in basic document clustering remains underexplored.
Traditionally, the translation of a document $\doc$ into a
representation $\query$ was considered a straightforward vectorization
step. However, texts are often sparse, with much that is unsaid and
based on implicit knowledge. By using LLMs, like GPT-4
\citep{openai2023gpt}, to generate $\query$ by complementing this
missing knowledge, could the content of $\doc$ be better explicated
and yield a better clustering result?

In this paper we explore the effect of using an LLM to rephrase each
document $\doc\in\docs$.  Previously, the term ``generative
clustering'' was used to refer to approaches where document vectors
appear to be ``generated'' from clusters
\citep{zhong2005generative,yang2022learning}. In this paper, however,
``generative'' refers to the use of a generative LLM.

The use of an LLM to translate $\doc$ into $\query$ has the further
benefit of opening the possibility to mathematically formulate the
information underlying $\doc$ more accurately. Previously, when
clustering a set of documents $\doc$ directly, not much could be done
statistically, given the bound by a finite number of words and the
lack of underlying knowledge. When using a simple representation such
as bag of words, document vectors are generated by a Gaussian
distribution for each cluster, thus reducing the clustering into a
simple $k$-means problem.  On the other hand, if we incorporate a set
of $\query$ produced by an LLM, the clustering problem can be
rigorously formulated using information-theory, according to the
probabilistic pairing information of $\doc\leftrightarrow\query$.

Specifically, $\query$ can be defined over the infinite set $\qspace$
of all possible word sequences of unlimited length. Each document is
represented by a probability distribution over generated texts in
$\qspace$, denoted as $p(\qvar=\query|\dvar=\doc)$. The dissimilarity
between documents is quantified using the KL divergence between their
respective distributions over the generated texts. Clustering then is
performed by grouping documents based on the similarity of these
distributions.

In this paper, we show that large improvements often result for plain
document clustering in this new generative clustering setting. This
suggests that the quantity of knowledge underlying a text is indeed
large, and that such knowledge can be inferred using generative LLMs.

This gain is already important even under the typical Gaussian
assumption that is often adopted when clustering $\dvar$ directly.
Clustering of $\qvar$, however, can go much further. For generative
clustering, the Gaussian assumption can be replaced by a more natural
distribution through importance sampling.  We empirically show
that our proposed method achieved consistent improvement on four
document-clustering datasets, often with large margins.

The proposed method's possible applications are vast, entailing all
problems that involve clustering. Here, we investigate the generative
document retrieval (GDR) \citep{tay2022dsi,wang2022nci}. In GDR,
documents are indexed using a prefix code acquired via hierarchical
clustering. We show that, on the MS Marco Lite dataset, GDR with the
proposed generative clustering (GC) method achieved up to 36\%
improvement in the retrieval accuracy.

\section{Related Works}
\subsection{Document Clustering}

Document clustering is a fundamental unsupervised learning task with
broad applications. Traditional methods depend on document
representations like bag of words (BoW), tf-idf, and topic models such
as latent Dirichlet allocation (LDA) \citep{blei2003latent}.

Recent advancements in text representation learning have significantly
improved clustering performance. Techniques like continuous
bag of words (CBoW) \citep{mikolov2013distributed} generate
low-dimensional word vectors, which can be averaged to form document
vectors. Deep neural networks have further enhanced these embeddings
by incorporating contextual information, as seen in models like ELMo
\citep{elmo}, BERT \citep{devlin2019bert}, and SentenceBERT
\citep{sbert2019}.

These methods focus on encoding documents as dense vectors that
capture similarities across texts. Studies have demonstrated the
effectiveness of pretrained models for clustering
\citep{subakti2022performance, guan2022deep}, along with further
improvements by using deep autoencoders and contrastive learning
\citep{xie2016unsupervised, guo2017improved}.

However, these approaches may not fully capture deeper knowledge
within texts, as BERT-like models are limited by their fixed vector
outputs and cannot complete complex tasks. In contrast, generative
models like GPT-4 \citep{openai2023gpt} can handle complex reasoning
through autoregressive text generation \citep{kojima2022large}, though
their use in clustering is still being explored.

Recent works have attempted to use generative LLMs for clustering, but
these approaches lack a rigorous foundation.
\citet{viswanathan2023large} explored the use of GPT models at several
clustering stages, including document expansion before vector
encoding, generation of clustering constraints, and correction of
clustering results. \citet{zhang2023clusterllm} used LLMs to predict a
document's relative closeness to two others and performed clustering
based on these triplet features. In contrast, our method formulates
the clustering problem in a rigorous information-theoretic framework
and leverages the LLM's ability to precisely evaluate generation
probabilities that were not used in these previous works.

\subsection{Generative Language Models}
\label{sec:pretrained-lm}

Generative language models are those trained to predict and generate
text autoregressively; they include both decoding-only models like GPT
and encoder-decoder models like T5 \citep{raffel2020exploring}. While
their architectures differ, both types of models function similarly,
and a GPT model can be used in a manner similar to the use of T5.

T5 models are typically trained in a multi-task setting, with
task-specific prefixes attached to input texts during training. In
Sections \ref{sec:eval-clustering} and \ref{sec:gdr}, we explore a
specific pre-trained T5 model called {\em doc2query}
\citep{nogueira2019doc2query}, which has been trained on various
tasks, including query generation, title generation, and question
answering.

\subsection{Information-Theoretic Clustering}

Information-theoretic clustering uses measures like the KL divergences
or mutual information as optimization objectives
\citep{dhillon2003information, slonim2000document}. As described in
Section \ref{sec:problem}, we aim to minimize the total KL divergence
between documents and their cluster centroids, where the centroids are
represented as distributions over the infinite set $\qspace$.

This approach is effective for finite discrete distributions, such as
word vocabularies, where the KL divergence can be computed exactly. In
continuous spaces, however, KL divergence estimation requires density
estimation, which involves assumptions about the data distribution and
limits applicability. Many studies have tackled these challenges
\citep{perez2008kullback, wang2009divergence, nguyen2010estimating}.

In our approach, $\qspace$ is an infinite discrete space where the
probability mass for each text can be evaluated exactly, though it
remains challenging to manage this infinite space. We address this
issue by using importance sampling \citep{kloek1978bayesian}, as
detailed in Section \ref{sec:algorithm}.

\section{Generative Clustering of Documents}
\label{sec:problem}

A clustering, in this work, is defined as a partition of documents
$\docs$ into $K$ clusters. Each document $\doc\in\docs$ is associated
with a conditional probability distribution $p(\qvar=\query|\doc)$,
which is defined for $\query$ over the infinite set $\qspace$ of all
possible word sequences. $p(\qvar=\query|\doc)$ measures the
probability mass that a text $\query$ is generated from a document
$\doc$, and this value is calculated by using a generative language
model with the following decomposition into word probabilities:
\begin{equation}
    p(\query|\doc) =
    p(\omega_1|\doc) \,p(\omega_2|\doc,\omega_1)\ldots p(\omega_{l}|\doc,\omega_1,\cdots,\omega_{l-1}),
    \label{eq:decomposition}
\end{equation}
where $\omega_1,\omega_2,\ldots,\omega_l$ are the words in the text
$\query$ and $l$ is the text's length.

Each cluster $k$ ($k=1,2,\cdots,K$) is also associated with a
distribution over $\qspace$, denoted as $p(\qvar|k)$. We refer to
$p(\qvar|k)$ as a cluster's ``centroid'' in this paper.

The clustering objective is to jointly learn a cluster assignment
function $f:\docs\to\{1,2,\cdots,K\}$ and the cluster centroids
$p(\qvar|k)$ while minimizing the total within-cluster distortion, as
follows:
\begin{align}
    \min_{f, \{p(\qvar|k)\}_k } D &= \sum_{\doc\in\docs} d(\doc, f(\doc)),
        \label{eq:objective} \\
    \text{where}~d(\doc, k) &=
    \mathrm{KL} \bigl[p(\qvar|\doc) \lVert p(\qvar|k)\bigr],
        \label{eq:distortion}
\end{align}
If $\qspace$ were finite, this objective is no more than a special
case of the Bregman hard clustering problem
\citep{banerjee2005clustering}. Note that $k$-means is a special case,
too. However, as mentioned at the beginning of this section, $\qspace$
{\em is infinite}; therefore, we need to figure out how to calculate
$p(\qvar|k)$. Unlike the document-specific distribution
$p(\qvar|\doc)$, which is acquired easily via an LLM, $p(\qvar|k)$ is
unknown. Moreover, we obviously need to determine how to calculate the
KL divergence for $\forall \doc,k$.

The condition probability $p(\qvar|\doc)$ can be seen as a novel kind of
document embedding for $\doc$. An embedding is a distribution over $\qspace$,
and is thus infinitely dimensional. By selecting a proper set of texts
$\queries$, $p(\qvar|\doc)$ can be ``materialized'' as a vector, with each entry
being $p(\qvar=\query|\doc)$ for a text $\query\in\queries$. However, the choice
of $\queries$ is crucial, and how to use the vector for clustering is not
straightforward. We answer these questions by proposing a novel clustering
algorithm in Section \ref{sec:algorithm}.

\section{A Conventional Baseline Using BERT}

Conventional clustering methods can be interpreted as a baseline
solution to the above problem, via approximating the KL divergence by
using a BERT model \citep{devlin2019bert}, for example, to embed the
generated texts $\qvar$ into a low-dimensional vector space. In other
words, each of $p(\qvar|\doc)$ and $p(\qvar|k)$ is a distribution over
the vector space.

Then, the clustering problem specified by Eqs.
(\ref{eq:objective}-\ref{eq:distortion}) is solved by using a standard
$k$-means algorithm and assuming the distributions $p(\qvar|\doc)$ and
$p(\qvar|k)$ are both Gaussian. Here, Eq. \eqref{eq:objective} reduces
to:
\begin{equation}
    \min_{f, \{\mathbb{E}[p(\qvar|k)]\}_k } D =
    \sum_{\doc\in\docs}
    \lVert \mathbb{E}[p(\qvar|\doc)] - \mathbb{E}[p(\qvar|f(\doc))\rVert^2,
    \label{eq:objective-euclid}
\end{equation}
where $f(x)$ is the cluster that $x$ is assigned to,
$\mathbb{E}$ represents the mean vector of a Gaussian
distribution. $\mathbb{E}[p(\qvar|\doc)]$ is estimated by the center
vector of texts sampled from $\doc$.

There are multiple methods to obtain $\mathbb{E}[p(\qvar|\doc)]$ and
we consider two here. The first is to let
$\mathbb{E}[p(\qvar|\doc)]=\text{BERT}(\doc)$, i.e., the document's
embedding. This method is equivalent to the most common clustering
method and does not use the generative language model. In contrast,
the second method uses the model by generating multiple texts for each
document $\doc$ and letting $\mathbb{E}[p(\qvar|\doc)]$ be the mean
vector of the generated texts' embeddings. The mean vector is used
because it is a maximum-likelihood estimator of
$\mathbb{E}[p(\qvar|\doc)]$.

Unfortunately, such conventional solutions involve multiple concerns.
It could be overly simplistic to assume that the generated texts are
representable within the same vector space and that $p(\qvar|\doc)$
and $p(\qvar|k)$ follow Gaussian distributions. In fact,
$p(\qvar|\doc)$ often exhibits multimodality when we use a BERT model
to embed the generated texts. Furthermore, this solution uses a text
embedding model in addition to the generative language model, and the
possible inconsistency between these two models might degrade the
clustering quality.

\section{Generative Clustering\\Using Importance Sampling}
\label{sec:algorithm}

\subsection{Two-Step Iteration Algorithm}
\label{sec:twostep}

Given the above concerns, we propose to perform generative clustering
by using importance sampling over $\qspace$.

First of all, we use the typical two-step iteration algorithm, as
generally adopted in $k$-means. This is the common solution for all
Bregman hard clustering problems \citep{banerjee2005clustering}, by
repeating the following two steps:
\begin{enumerate}
    \item Assign each document $\doc$ to its closest
    cluster by distance $d(\doc,k)$.
    \item Update every cluster centroid to minimize the
    within-cluster total distance from the documents to the centroid.
\end{enumerate}
We formally implement this as Algorithm \ref{alg:clustering}.
The $d(\doc,k)$ function is defined in the following subsection.  The
algorithm starts in line 1 by sampling a set of texts $\queries =
\{\query_1,\query_2,\cdots,\query_J \}$ from the LLM, as explained
further in Section \ref{sec:proposal}.  Then, the algorithm computes
two matrices, which both have rows corresponding to documents and
columns corresponding to sampled texts in $\queries$. The first matrix
(line 2) is the probability matrix of $\mathbf{P}_{ij}\equiv
p(y_i|x_i)$, acquired from the LLM. After a {\em clipping} procedure
for $\mathbf{P}_{ij}$ (line 3, Section \ref{sec:clipping}) and
estimation of the {\em proposal} distribution $\phi(\query_j)$ (line
4, Section \ref{sec:proposal}), the second matrix (lines 4,5) is acquired
as regularized importance weights $\mathbf{W}$ such that $\mathbf{W}_{ij}
\equiv (p(\query_j|\doc_i) / \phi(\query_j))^\alpha$, where
$\alpha$ is a hyperparameter explained in the following section. Next,
the clustering algorithm, defined in lines 8-15, is called (line 6).
After centroid initialization (line 9, Section \ref{sec:centroid}),
the two steps described above are repeated (lines 11 and 12-13,
respectively).

Algorithm \ref{alg:clustering} is guaranteed to converge (to a local
minimum), and we provide a proof in Appendix \ref{sec:converge}
(Proposition \ref{thm:converge}).

\begin{algorithm}[tb]
\small
\caption{Generative clustering of documents.}
\label{alg:clustering}
\textbf{Input}: Documents $\docs$; a finetuned language model. \\
\textbf{Parameters}: Number of clusters, $K$; number of text samples,
$J$; proposal distribution, $\phi$; regularization factor, $\alpha$. \\
\textbf{Output}: Cluster assignment function $f:\docs\to\{1,2,\cdots,K\}$.

\begin{algorithmic}[1]
    \State \multilinecol{Generate $\queries=\{\query_1,\cdots,\query_J\}$
        by i.i.d. sampling from the language model. }{}
    \State \multilinecol{
        Compute $\mathbf{P}_{ij}=p(\query_j|\doc_i)$ for
        any $\doc_i\in\docs$ and $\query_j\in\queries$ via the
        language model.
        }{
        Eq.\eqref{eq:decomposition}
        }
    \State \multilinecol{
        Clip $\mathbf{P}_{ij}$ to avoid outlier values.
    }{Eq.\eqref{eq:clip}}
    \State \multilinecol{
        For $\query_j\in \queries$, estimate $\phi(\query_j)$.
        }{
        Eq.\eqref{eq:vmp}
        }
    \State \multilinecol{Compute the importance weight matrix\\
        $
        \mathbf{W}_{ij} = (\mathbf{P}_{ij} / \phi(\query_j))^\alpha$.
      }{}
    \State $f \gets \Call{Clustering}{\mathbf{P}, \mathbf{W}}$ \\

    \Function{Clustering}{$\mathbf{P}, \mathbf{W}$}
    \State \multilinecol{
        Initialize cluster centroids $\mathbf{c}_k$ ($\forall k$) by
        randomly selecting a row of $\mathbf{W}$ and normalizing it.
        }{Eq.\eqref{eq:init}}
    \Do
        \State \multilinecol{
            $f(\doc_i)\gets$ closest cluster to $\doc_i$,
            $\forall\doc_i\in\docs$.
            }{Eq.\eqref{eq:ris} }
        \State \multilinecol{
            Update the centroid $\mathbf{c}_k$.
            }{Eq.\eqref{eq:optim-centroid}}
        \State \multilinecol{
            Compute the total distortion \\
            $\sum_{\doc\in\docs} \hat{d}(\doc,f(\doc))$.
            }{Eq.\eqref{eq:ris}}
    \doWhile{the total distortion improves.}
    \State \Return $f$
    \EndFunction
\end{algorithmic}
\end{algorithm}

\subsection{Distortion Function $d(\doc,k)$}
\label{sec:distortion}

The distortion function $d(\doc, k)$ was defined in Eq.
\eqref{eq:distortion} as the KL divergence between $p(\qvar|\doc)$ and
$p(\qvar|k)$, but it is not computable because of the infinite
$\qspace$. Therefore, we apply the {\em importance sampling} (IS)
technique \citep{kloek1978bayesian} to estimate $d$. IS is a
statistical technique to estimate properties of a particular
distribution from a different distribution $\phi$ called the {\em
proposal}. The technique's effectiveness is demonstrated by the
following equality:
\begin{equation*}
    \mathbb{E}_{\qvar\sim p(\qvar|\doc)}\left[\log\frac{p(\qvar|\doc)}{p(\qvar|k)}\right]
    = \mathbb{E}_{\qvar\sim\phi}\left[
      \frac{p(\qvar|\doc)}{\phi(\qvar)}
      \log\frac{p(\qvar|\doc)}{p(\qvar|k)}
      \right],
\end{equation*}
where the left side is the KL divergence, and
$\frac{p(\qvar|\doc)}{\phi(\qvar)}$ on the right side is called
the {\em importance weight}. By using IS, the KL divergence for each
document can be estimated efficiently by using samples generated from
a shared $\phi$ (i.e., $\queries$ in line 1 of Algorithm
\ref{alg:clustering}) across all documents, instead of a separate
distribution $p(\qvar|\doc)$ for each document.

A long-standing problem of IS-based estimators is the potentially
large variance if the proposal $\phi$ is poorly selected and distant
from the original distribution $p(\qvar|\doc)$. In addition to
carefully selecting $\phi$, as detailed in Section \ref{sec:proposal},
we adopt a {\em regularized} version of IS (RIS)
\citep{korba2022adaptive}, which imposes an additional power function
(with parameter $\alpha$) on the importance weights. Our estimator
$\hat{d}$ based on RIS is defined as follows:
\begin{equation}
    \hat{d}(\doc,k) \equiv \frac{1}{J}\sum_{\substack{j=1 \\ \query_j\sim \phi(\qvar)}}^J \left(
        \frac{p(\query_j|\doc)}{\phi(\query_j)}
    \right)^\alpha \log\frac{p(\query_j|\doc)}{p(\query_j|k)}.
    \label{eq:ris}
\end{equation}
Here, $\alpha\in[0,1]$ defines the strength of regularization. When
$\alpha=1$, Eq. \eqref{eq:ris} converges to the true KL divergence
value as $J\to\infty$; otherwise, Eq. \eqref{eq:ris} is biased, and
smaller $\alpha$ leads to more bias. $\alpha$ plays a central role in
the experimental section and is further discussed after that, in
Section \ref{sec:biasness}. In this paper, $\alpha$ was set to 0.25
for all experiments.
Notably, there are other choices for importance weight regularization, e.g.,
\citet{aouali2024unified}.

Using the matrices defined in Section \ref{sec:twostep},
this function $\hat{d}$ is rewritten as follows:
\begin{equation}
\hat{d}(\doc_i, k) = \frac{1}{J} \sum_{j=1}^J \mathbf{W}_{ij} \log \frac{\mathbf{P}_{ij}}{\mathbf{c}_k(j)},
\end{equation}
where $\mathbf{c}_k$ is a vector representing the cluster centroid $p(\qvar|k)$,
as defined later in Eq. \eqref{eq:centroid}. The weight matrix $\mathbf{W}$
assigns importance to each text $\query_j$, typically favoring texts with
higher probabilities (e.g., shorter texts).
Unlike the Euclidean distance, which treats all dimensions equally, the weight
matrix $\mathbf{W}$ introduces a preference for specific dimensions (i.e.,
specific texts in $\queries$), enabling the algorithm to optimize an
information-theoretic quantity.

\subsection{Proposal Distribution $\phi(\qvar)$}
\label{sec:proposal}

The proposal distribution $\phi(\qvar)$ in Eq. \eqref{eq:ris} serves
two functions: (1) construction of
$\queries=[\query_1,\cdots,\query_J]$ by sampling $\query_j$ from
$\phi$, and (2) estimation of the probability mass $\phi(\query_j)$
for any text $\query_j\in\queries$.

The idea of tuning the proposal $\phi$ for better estimation of some
quantity of a single distribution is not novel. In our setting,
however, we must estimate $d(\doc,k)$ for all $\doc$ and $k$, i.e.,
for multiple distributions. Therefore, we must carefully choose
$\phi$, instead of using some standard distribution.

In this paper, we choose the prior distribution $p(\qvar)$ as the
proposal $\phi(\qvar)$, for two reasons. First, $p(\qvar)$ is close to
$p(\qvar|\doc)$ for all $\doc\in\docs$, thus minimizing the RIS
sampler's overall variance across all documents. Second, it allows for
a straightforward sampling process:
\begin{description}
  \item[\rm(i)] Randomly select a document $\doc$ from $\docs$ uniformly.
  \item[\rm(ii)] Generate a text $\query_j$ from $\doc$ by using the
  language model.
\end{description}
Here, the documents $\docs$ are assumed to be i.i.d. samples from the
prior distribution $p(\dvar)$.

For estimating $p(\query_j)$, a basic approach is to average
$p(\query_j|\doc)$ across all $\doc\in\docs$, which is equivalent to
averaging the $j$-th column of matrix $\mathbf{P}$. However, we have
found that the following estimator works better with Algorithm
\ref{alg:clustering}:
\begin{equation}
  \phi(\query_j)=p(\query_j) \gets \left(
  \frac{1}{|\docs|} \sum_{\doc\in\docs} p(\query_j|\doc)^{2\alpha}
  \right)^{1/2\alpha}, \label{eq:vmp}
\end{equation}
where $\alpha$ is the regularization parameter in Eq. \eqref{eq:ris}.
This function in Eq. \eqref{eq:vmp} represents the optimal case for
minimal variance. The function's theoretical background and
approximation are discussed in Appendix \ref{sec:optim-phi}
(Proposition \ref{thm:phi}).

\subsection{Estimation of Cluster Centroid $p(\qvar|k)$.}
\label{sec:centroid}

Each cluster is specified by its ``centroid'' $p(\qvar|k)$, which is a
probability distribution defined over the infinite space $\qspace$. In
our algorithm, $p(\qvar|k)$ is represented as a vector,
\begin{equation}
    \mathbf{c}_k = [p(\query_1|k), \cdots, p(\query_J|k)],
    \label{eq:centroid}
\end{equation}
comprising the values of $p(\qvar|k)$ on $\queries$, which is a subset
of the infinite $\qspace$. The total probability on $\queries$ is
denoted by $C_k=\sum_{j} p(\query_j|k) \leq 1$. Here, $C_k$ is a
hyperparameter to adjust the assignment preference for cluster $k$, as
increasing $C_k$ would reduce the distance (i.e., the KL divergence)
from any document to cluster $k$. Because we have no prior knowledge
about each cluster, we set $C_1=C_2=\cdots=C_K=:C$ in this paper,
meaning there is no preference among clusters. Furthermore, the order
of the distances from a document to all cluster centroids is preserved
under variation of $C$, and we thus set $C=1$ for simplicity.

For initialization, $\mathbf{c}_k$ is set to a random row of
$\mathbf{W}$ and normalized to sum to 1. Let $i_k$ denote the
selected row's index. Then, every entry of $\mathbf{c}_k$ is
initialized as follows:
\begin{equation}
    \mathbf{c}_k(j) \gets \mathbf{W}_{i_k j} /
    \sum_{j'=1}^J \mathbf{W}_{i_k j'}.
    \label{eq:init}
\end{equation}
We also tried a more sophisticated initialization method by adapting the
$k$-means++ strategy\citep{arthur2006kmeanspp}. However, the
results were not significantly different on the datasets we used. We
present the results in Appendix \ref{sec:kmeanspp}.

After the assignment step in each iteration, the cluster centroids are
updated. For cluster $k$, the optimal centroid vector $\mathbf{c}_k^*$
that minimizes the within-cluster total distortion is parallel to the
mean vector of the rows in $\mathbf{W}$ for the documents belonging to
this cluster. That is, at each iteration, every cluster centroid is
updated as follows:
\begin{equation}
    \mathbf{c}_k(j) \gets \mathbf{c}_k^*(j) = \frac{1}{Z}
    \sum_{\doc_i\in f^{-1}(k)} \mathbf{W}_{ij},
    \label{eq:optim-centroid}
\end{equation}
where $Z=\sum_{j=1}^J \mathbf{c}_k^*(j)$ is the normalization term.
The optimality is shown in Appendix \ref{sec:optim-centroid}
(Proposition \ref{thm:centroid}).

\subsection{Probability Clipping}
\label{sec:clipping}

Probability clipping is common in importance sampling
\citep{wu2020improving} to stabilize the estimation procedure, by
eliminating outlier values in $\mathbf{P}$. Outliers typically occur
in $\mathbf{P}_{ij}$ when $\query_j$ is generated from $\doc_i$, and
they can be several magnitudes larger than other values in the same
column, which destabilize KL divergence estimation. We use the
following criterion to reset these outliers in the logarithm domain:
\begin{equation}
    \log\mathbf{P}_{ij} = \begin{cases}
    \mu_j + 5\sigma_j & \text{if}~\log \mathbf{P}_{ij} > \mu_j + 5\sigma_j, \\
    \log\mathbf{P}_{ij} & \text{otherwise},
    \end{cases}
    \label{eq:clip}
\end{equation}
where $\mu_j$ and $\sigma_j$ are the mean and standard deviation of
the $j$-th column of $\log\mathbf{P}_{ij}$.
This clipping step can be seen as a further regularization of the
importance weights \citep{aouali2024unified}.

\section{Evaluation of Generative Clustering}
\label{sec:eval-clustering}

\subsection{Data}

We conducted an evaluation of our method using four document
clustering datasets: R2, R5, AG News, and Yahoo! Answers, as
summarized in Table \ref{tbl:dataset}.

\begin{table}[bp]
\centering
\small
\begin{tabularx}{\linewidth}{@{\extracolsep{\fill}} lrrr}
    \toprule
    Dataset & \# documents & mean \# words & \# clusters \\
    \midrule
    R2 & 6,397 & 92.4 & 2 \\
    R5 & 8,194 & 116.4 & 5 \\
    AG News & 127,600 & 37.8 & 4 \\
    Yahoo! Answers & 1,460,000 & 69.9 & 10 \\
    \bottomrule
\end{tabularx}
\caption{
Document clustering datasets used for clustering.}
\label{tbl:dataset}
\end{table}

\begin{table*}[t]
\centering
\small
\begin{tabularx}{\linewidth}{@{\extracolsep{\fill}} lcccccccccccc}
    \toprule
    Method & \multicolumn{3}{c}{R2} & \multicolumn{3}{c}{R5} & \multicolumn{3}{c}{AG News} & \multicolumn{3}{c}{Yahoo! Answers}  \\
    \cmidrule{2-4} \cmidrule{5-7} \cmidrule{8-10} \cmidrule{11-13}
         & ACC & NMI & ARI & ACC & NMI & ARI & ACC & NMI & ARI & ACC & NMI & ARI \\
    \midrule
    \multicolumn{13}{c}{\em $k$-means clustering} \\
    \multicolumn{4}{l}{\em Non-neural document embeddings} \\
    Bag of words                       & 65.0 &  3.3 &  5.9 & 40.1 & 20.0 & 16.1 & 29.4 & 1.68 &  0.7 & 15.2 &  3.1 &  1.0 \\
    tf-idf                             & 77.7 & 40.2 & 31.1 & 48.6 & 33.2 & 18.6 & 43.0 & 15.9 & 15.5 & 18.9 &  6.0 &  2.5 \\
    LDA \citep{blei2003latent}         & 85.3 & 49.9 & 48.5 & 61.7 & 40.4 & 37.2 & 45.1 & 28.4 & 22.2 & 21.8 & 10.0 &  4.5 \\
    \hdashline[0.5pt/1pt] \noalign{\vskip 0.5ex}
    \multicolumn{13}{l}{\em BERT/SBERT document embeddings (different pretrained models)} \\
    Bert (base, uncased)      & 75.6 & 31.1 & 28.1 & 58.2 & 27.3 & 35.0 & 59.9 & 38.2 & 30.9 & 20.9 & 9.4 & 4.5 \\
    Bert-base-nli-mean-tokens & 86.7 & 49.6 & 53.7 & 57.1 & 36.3 & 36.7 & 64.1 & 29.8 & 28.2 & 31.9 & 15.1 & 10.7 \\
    All-minilm-l12-v2         & 86.9 & 54.3 & 54.5 & 76.7 & 65.6 & 58.6 & 74.4 & 56.7 & 55.5 & 58.2 & 42.8 & 34.8 \\
    All-mpnet-base-v2         & 90.3 & 61.7 & 65.0 & 73.2 & 68.9 & 58.5 & 69.2 & 55.0 & 52.0 & 56.7 & 39.8 & 32.4 \\
    All-distilroberta-v1      & 92.0 & 65.6 & 70.5 & 69.7 & 65.2 & 62.4 & 67.0 & 56.1 & 51.9 & 57.5 & 42.6 & 34.2 \\
    \hdashline[0.5pt/1pt] \noalign{\vskip 0.5ex}
    \multicolumn{4}{l}{\em Doc2query-based document embeddings} \\
    Doc2Query encoder                  & 74.2 & 32.5 & 24.8 & 48.6 & 24.2 & 18.7 & 38.1 & 8.4 & 7.6 & 16.5 & 4.0 & 1.5 \\
    Doc2Query log-prob. $\mathbf{P}$ & 86.1 & 51.5 & 52.3 & 56.2 & 54.1 & 43.9 & 74.8 & 53.2 & 52.1 & 43.0 & 29.6 & 20.7 \\
    \midrule
    \multicolumn{13}{c}{\em Other non-$k$-means clustering methods} \\
    GSDMM\textsuperscript{$\dagger$} \citep{yin2014dirichlet} & 74.5 & 62.3 & 45.3 & 57.6 & 59.5 & 44.5 & 43.5 & 42.5 & 30.4 & 49.6 & 36.1 & 34.2 \\
    DEC\textsuperscript{$\dagger$} {\scriptsize\citep{xie2016unsupervised}} & 76.0 & 37.4 & 36.9 & 37.9 & 22.2 & 18.5 & 47.0 & 12.4 & 11.5 & 14.2 & 1.2 & 1.1 \\
    IDEC\textsuperscript{$\dagger$} \citep{guo2017improved}   & 77.7 & 32.3 & 31.5 & 46.1 & 34.8 & 29.5 & 63.5 & 29.5 & 29.4 & 22.0 & 10.3 & 9.0 \\
    STC\textsuperscript{$\dagger$}  \citep{xu2017self}        & 77.9 & 35.7 & 36.5 & 57.1 & 40.8 & 37.5 & 64.6 & 39.2 & 38.8 & 37.7 & 23.6 & 21.4 \\
    DFTC\textsuperscript{$\dagger$} \citep{guan2022deep}      & 84.7 & 50.4 & 49.0 & 69.6 & 61.5 & 61.5 & 82.1 & 60.3 & 60.0 & 51.1 & 35.0 & 34.4 \\
    \midrule
    \multicolumn{13}{c}{\em Generative clustering (ours)} \\
    GC ($\alpha=0.25$) & \textbf{96.1} & \textbf{77.8} & \textbf{84.9} & \underline{79.1} & \textbf{71.5} & \textbf{69.1} & \textbf{85.9} & \textbf{64.2} & \textbf{67.2} & \textbf{60.7} & \textbf{43.7} & \underline{35.5} \\
      & (0.0) & (0.1) & (0.1) & (2.9) & (2.1) & (2.9) & (0.6) & (0.8) & (1.2) & (0.7) & (0.2) & (0.3)   \\
    \hdashline[0.5pt/1pt] \noalign{\vskip 0.5ex}
    \multicolumn{13}{l}{\em Ablated versions} \\
    - $\alpha=1$ (unbiased KL estimation) & 78.0 & 25.8 & 33.4 & 57.3 & 31.1 & 33.7 & 55.6 & 22.9 & 22.0 & 34.5 & 18.7 & 11.0 \\
    - Na\"ive $p(\qvar)$ estimator  & \underline{95.9} & \underline{77.1} & \underline{84.2} & \textbf{79.3} & \underline{71.1} & \underline{68.5} & \underline{85.0} & \underline{63.2} & \underline{65.5} & \underline{60.5} & \underline{43.4} & \textbf{35.6} \\
    - No probability clipping & 95.1 & 73.2 & 81.2 & 76.4 & 65.6 & 63.0 & 80.6 & 60.7 & 59.4 & 57.5 & 42.3 & 33.7 \\
    \bottomrule
\end{tabularx}
\caption{
Clustering performance of various methods (rows) on four document
datasets (columns). Methods marked with $\dagger$ in the first column
were taken from \citet{guan2022deep}. The values in parentheses are
standard deviations across 100 repeated experiments. The best score in
each column is in bold and the second best is underlined. }
\label{tbl:clustering}
\end{table*}

R2 and R5 are subsets of Reuters-21587 and respectively contain
documents from the largest two (\textsc{earn}, \textsc{acq}) and
largest five (also \textsc{crude}, \textsc{trade}, \textsc{money-fx})
topics, following \citet{guan2022deep}. For AG News \citep{agnews}, we
used the version provided by \citet{zhang2015character}. Yahoo!
Answers is a more challenging dataset with more documents and
clusters.

For R2 and R5, we discarded documents tagged ambiguously with multiple topics.
For AG News and Yahoo! Answers, we merged the training and test splits, as our
method is unsupervised.

\subsection{Evaluation Metrics \& Settings}

Following previous works \citep{guan2022deep,zhou2022comprehensive},
we evaluated clustering methods by using three metrics:
\begin{description}
\item[Accuracy (ACC):] The percentage of correct cluster assignments,
 accounting for label permutations by maximizing the accuracy over all
 permutations with the Jonker-Volgenant
 algorithm (implemented in Python via \\
 \texttt{\small scipy.optimize.linear\_sum\_assignment}).
\item[Normalized Mutual Information (NMI):] A measure of the mutual
 information between true and predicted labels, normalized by the
 geometric mean of their Shannon entropies \citep{nmi}.
\item[Adjusted Rand Index (ARI):] A variant of the Rand index
 that adjusts for random labeling and reduces the sensitivity to the
 number of clusters \citep{hubert1985comparing}.
\end{description}

Clustering algorithms are sensitive to initialization and affected by
randomness. To mitigate this, we performed model selection by running each
method 10 times with different seeds and selecting the run with the lowest total
distortion. We repeated this process 100 times, and report the mean performance
of the selected runs. For $k$-means, the distortion is the sum of the squared
distances from points to their centroids. For our method, the distortion is
$\sum_{\doc\in\docs} \hat{d}(\doc, f(\doc))$, where $\hat{d}$ is defined in Eq.
\eqref{eq:ris} and $f$ is the cluster assignment function. This model selection
step uses no label information and is standard in practice.

For the language model, we used the pretrained {\em doc2query} model
\textrm{\small
all-with\_prefix-t5-base-v1}.\footnote{\url{https://huggingface.co/doc2query/all-with_prefix-t5-base-v1}}
This model, trained on multiple tasks, attaches a task-specific prefix
to each document. We used the general prefix ``text2query:'' to
generate $\queries$, and we evaluated the generation probabilities
from $\docs$ to $\queries$. We set $\alpha$ to 0.25 and $J$ to 1024 by
default. We set $K$ to the number of clusters in each dataset, as seen
in the rightmost column of Table \ref{tbl:dataset}.

For BERT-based baselines, we tested the original BERT model
\citep{devlin2019bert} and multiple SBERT models \citep{sbert2019}
that were fine-tuned for document representation and clustering. The
pre-trained SBERT models are available at
\url{https://sbert.net/docs/sentence_transformer/pretrained_models.html}.

\subsection{Results}

The clustering results are summarized in Table \ref{tbl:clustering},
with the rows representing methods or models and the columns
representing datasets and evaluation metrics. The methods are
categorized into three groups: $k$-means (top), non-$k$-means
(middle), and our generative clustering (GC, bottom).

Our method consistently outperformed the others across all datasets,
often by significant margins. For instance, on the R2 dataset, GC
achieved 96.1\% accuracy, reducing the error rate from 8.0\% to 3.9\%.
The NMI and ARI scores also improved to 77.8 (from 65.6) and 84.9
(from 70.5), respectively. This advantage extended across the
datasets, ranging from R2 (under 10K documents) to Yahoo! Answers
(over 1 million).

While $k$-means clustering on SBERT embeddings sometimes approached
GC's performance, such as with \textrm{\small all-distilroberta-v1}
on R2, GC remained the top performer across all datasets without
fine-tuning.

In the table, rows 9 and 10 involve the same doc2query model as GC but
applied it differently. For row 9, $k$-means was applied to embeddings
from the model's encoder, similar to document processing with BERT.
For row 10, $k$-means was applied to $\log\mathbf{P}$ generated by the
doc2query decoder.

Comparison of row 9 with GC highlights the gains from translating
documents $\docs$ into $\queries$ with an LLM. Although both methods
used the doc2query model, the row-9 results were only comparable to
BERT's performance, while GC performed significantly better.

Row 10 shows improved results, nearing those of SBERT-based methods,
despite $k$-means being less suited for log-probabilities. This
demonstrates the potential of using generation probabilities for
clustering. Furthermore, GC's performance margin over the row-10
results demonstrates the advantage of not relying on the Gaussian
assumption underlying $k$-means but instead using a discrete
distribution over $\qspace$ to better approximate the data's true
distribution.

\begin{figure*}[htbp]
\centering
\small
\begin{minipage}[t]{0.22\linewidth}
    \centering
    ~~~~~~R2
    \includegraphics[width=\linewidth]{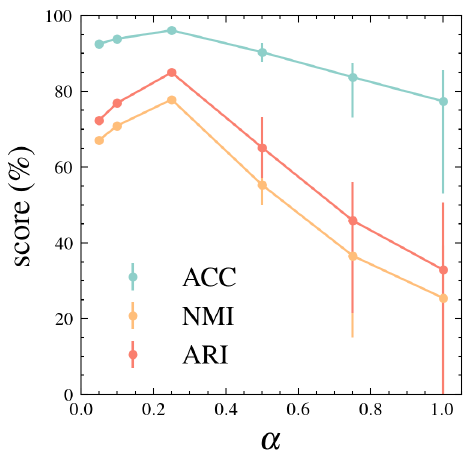} \\
    ~~~~~~(a)
\end{minipage}
\begin{minipage}[t]{0.22\linewidth}
    \centering
    ~~~~~~R5
    \includegraphics[width=\linewidth]{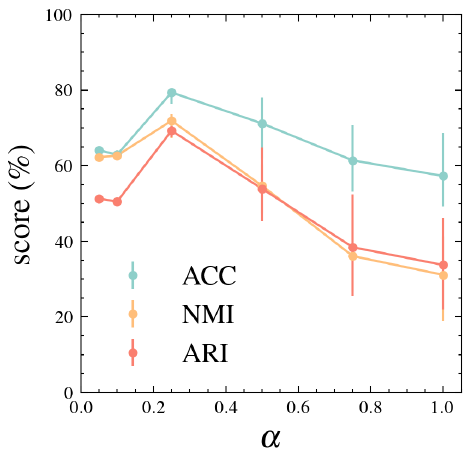} \\
    ~~~~~~(b)
\end{minipage}
\begin{minipage}[t]{0.22\linewidth}
    \centering
    ~~~~~~AG News
    \includegraphics[width=\linewidth]{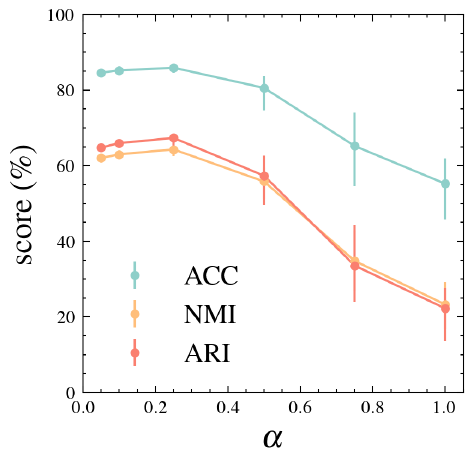} \\
    ~~~~~~(c)
\end{minipage}
\begin{minipage}[t]{0.22\linewidth}
    \centering
    ~~~~~~Yahoo! Answers
    \includegraphics[width=\linewidth]{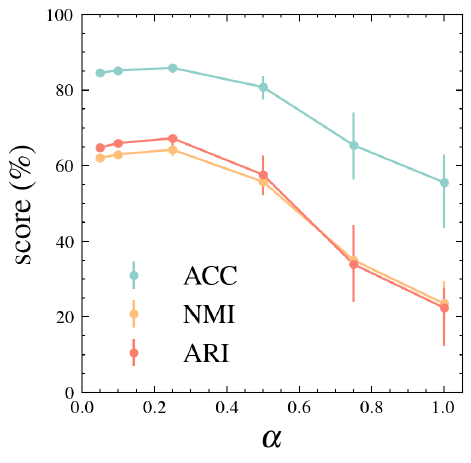} \\
    ~~~~~~(d)
\end{minipage}
\begin{minipage}[t]{0.22\linewidth}
    \centering
    \includegraphics[width=\linewidth]{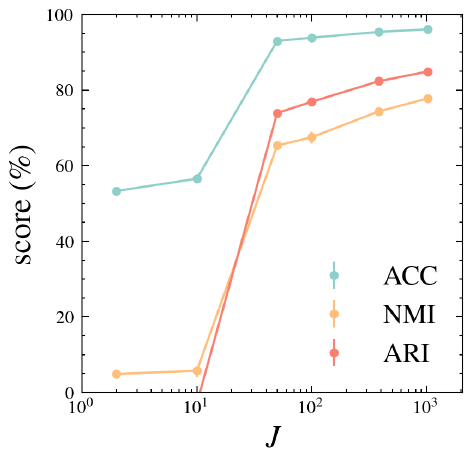} \\
    ~~~~~~(e)
\end{minipage}
\begin{minipage}[t]{0.22\linewidth}
    \centering
    \includegraphics[width=\linewidth]{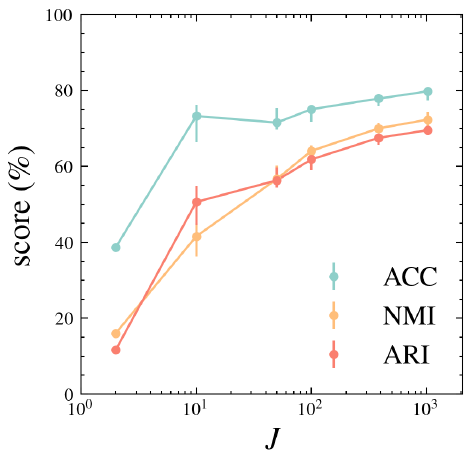} \\
    ~~~~~~(f)
\end{minipage}
\begin{minipage}[t]{0.22\linewidth}
    \centering
    \includegraphics[width=\linewidth]{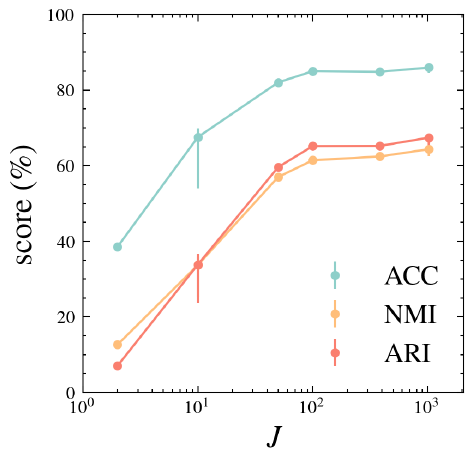} \\
    ~~~~~~(g)
\end{minipage}
\begin{minipage}[t]{0.22\linewidth}
    \centering
    \includegraphics[width=\linewidth]{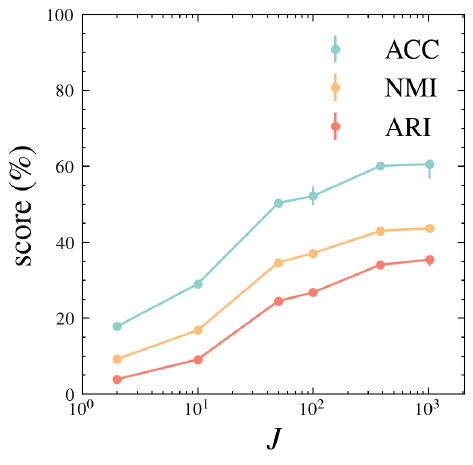} \\
    ~~~~~~(h)
\end{minipage}
\caption{
Performance of GC on the four datasets (columns) with varying (a-d)
$\alpha$ values or (e-h) $J$ values. Different colors represent
different clustering evaluation metrics, with error bars indicating
the 95\% confidence interval, based on 100 repeated experiments. Error
bars for some $\alpha$ and $J$ values are too small to be visible. }
\label{fig:alpha-J}
\end{figure*}

\paragraph{Ablated versions.}
The bottom of Table \ref{tbl:clustering} shows the results for three
ablated versions of GC. These versions used $\alpha=1$ instead of the
recommended $\alpha=0.25$, estimated the proposal distribution
$p(\qvar)$ with a na\"ive estimator instead of Eq. \eqref{eq:vmp}, and
omitted the probability clipping step in Eq. \eqref{eq:clip}.

The largest performance loss occurred with $\alpha=1.00$ instead of
$\alpha=0.25$, an interesting outcome that we discuss in Section
\ref{sec:biasness}. Results for other
$\alpha$ values are given in Figures \ref{fig:alpha-J}(a-d) for the four
datasets. $\alpha=0.25$ consistently yielded the best performance across all
datasets. The plots represent mean values over 100 repeated experiments, with
error bars indicating the 95\% confidence interval.

Figures \ref{fig:alpha-J}(e-h) show the performance as a function of the sample
size $J$ for the four datasets. The plots represent mean values from 100
repeated experiments, with error bars indicating 95\% confidence intervals.
Increasing $J$ requires generating and evaluating more texts, but it yields
higher clustering scores and more stable outcomes. On smaller datasets (R2, R5
and AG News), performance curves converged for $J\geq 50$. For the largest
dataset, Yahoo! Answers, convergence was evident at $J\geq 384$, which is
smaller than the typically dimensionality of BERT embeddings (768).

The probability clipping step also brought consistent improvement
across all datasets. As for using a na\"ive estimator for $p(\qvar)$,
the performance decrease was less severe, but still noticeable on
the R2 and the AG News datasets.

\paragraph{Clustering with misspecified $K$} We also evaluated GC with
misspecified numbers of clusters, $K$. Results are summarized in Appendix
\ref{sec:misspec-K}. With $K$ misspecified, GC and SBERT-based methods suffered
performance degradation, but GC still steadily outperformed SBERT-based methods
across all tested $K$ values from 2 to 20. This suggests the robustness of GC to
the choice of $K$.

\section{Application to Generative Retrieval}
\label{sec:gdr}

Among vast application possibilities, GC can be naturally applied to
generative document retrieval (GDR), \citep{tay2022dsi}.  In GDR, the
$\queries$ are queries for the retrieval, and documents $\docs$ are
indexed by using prefix coding, which necessitates hierarchical
clustering.

\subsection{Hierarchical Clustering with Localized $\phi$}
\label{sec:hierarchical}

\begin{algorithm}[htbp]
\caption{Hierarchical clustering.}
\label{alg:hierarchical}
\small
\textbf{Input}:
    Documents $\tilde{\docs}\subset \docs$ in a sub-cluster;
    the generated texts $\queries=\{ \query_1,\cdots,\query_J \}$
    and the probability matrix $\mathbf{P}$ produced
    in Algorithm \ref{alg:clustering} (for the whole $\docs$).
    \\
\textbf{Parameters}: Number of clusters, $K$; number of text
samples, $J$; proposal distribution, $\phi(\qvar)$;
regularization factor, $\alpha$. \\
\textbf{Output}: Cluster assignment function
$f:\docs\to\{1,2,\cdots,K\}$.
\begin{algorithmic}[1]
    \State \multilinecol{
        For $\query_j\in \queries$,
        estimate $\tilde{\phi}(\query_j)$ using only the documents in
        $\tilde{\docs}$.
        }{Eq.\eqref{eq:vmp}}
    \State \multilinecol{
        Calculate the resampling weights $\mathbf{r}_{j}$.
        }{Eq.\eqref{eq:resampling}}
    \State
        \multilinecol{
        Resample $J$ texts from $\queries$ according to the weights
        $\mathbf{r}_j$, forming $\tilde{\queries}$ and
        correspondingly, $\tilde{\mathbf{P}}$, which is a reselection
        of $\mathbf{P}$'s columns with replacement.
        }{}
    \State \multilinecol{Compute the importance weight matrix\\
        $\tilde{\mathbf{W}}_{ij} = \left(\tilde{\mathbf{P}}_{ij}
        / \tilde{\phi}(\query_j) \right)^\alpha$.}{}
    \State $f \gets \Call{Clustering}{\tilde{\mathbf{P}},
        \tilde{\mathbf{W}}}$ from Algorithm \ref{alg:clustering}.
\end{algorithmic}
\end{algorithm}

The clustering in Algorithm \ref{alg:clustering} can be extended to a
hierarchical version, given as Algorithm \ref{alg:hierarchical}. When
clustering documents within a sub-cluster $\tilde{\docs} \subset
\docs$, the proposal distribution $\phi$ is localized to this
sub-cluster for improved KL divergence estimation and denoted as
$\tilde{\phi}$. This localized distribution is estimated using Eq.
\eqref{eq:vmp} with sub-cluster documents (see Appendix
\ref{sec:optim-phi} for the details).

We apply bootstrapping to generate samples from $\tilde{\phi}$ without
recalculating $\mathbf{P}$. Resampling is performed from
$\queries=[\query_1,\cdots,\query_J]$ with weights $\mathbf{r}$:
\begin{equation}
    \mathbf{r}_{j} =
    \left(\dfrac{\tilde{\phi}(\query_j)}{\phi(\query_j)}\right)^\alpha
    \label{eq:resampling}
\end{equation}

Repeating this resampling $J$ times generates a new text set
$\tilde{\queries}$, with its corresponding localized probability
matrix $\tilde{\mathbf{P}}$ and importance weights
$\tilde{\mathbf{W}}$ derived from $\mathbf{P}$ and $\tilde{\phi}$.
Then, $\tilde{\docs}$ is clustered by the iterative procedure in
Algorithm \ref{alg:clustering} with these localized matrices.

\subsection{Document Indexing Through Hierarchical Clustering}
\label{sec:indexing}

In GDR, documents are indexed using a prefix code generated by
hierarchical clustering. Each document is assigned a unique numerical
string based on its cluster index at each level of the hierarchy.

Previous methods used BERT-based vectors and $k$-means for this prefix
code construction, as in DSI \citep{tay2022dsi} and NCI
\citep{wang2022nci}. Recently, \citet{du2024bmi} proposed query-based
document clustering, but using BERT and $k$-means clustering, which
should be improved upon by our method.

After constructing the prefix code, a neural retrieval model is
trained to map queries to the numerical strings of their corresponding
documents. The clustering method's effectiveness was evaluated via the
retrieval accuracy. We used the same setting as in \citet{wang2022nci}
for the retrieval model and the training details. In other words, the
only change was the construction of document indexes with our GC
instead of $k$-means. Our indexing method is abbreviated as GCI. We
set $K=30$ and $J=4096$.

We compared our method with NCI and BMI, which use hierarchical
$k$-means clustering for document indexing. We used a small retrieval
model with about 30M parameters to highlight the clustering method's
advantage. Two datasets, NQ320K \citep{kwiatkowski2019natural} and MS
Marco Lite \citep{du2024bmi}, were used for evaluation, with the
numbers of documents listed in the top row of Table
\ref{tbl:gdr-mini}. We fine-tuned the doc2query models on the two datasets.

\subsection{Results}

\begin{table}[tbp]
\centering
\small
\begin{tabularx}{\linewidth}{lcccc}
    \toprule
        & \multicolumn{2}{c}{MS Marco Lite } & \multicolumn{2}{c}{NQ320K} \\
        & \multicolumn{2}{c}{(138,457 docs.)} & \multicolumn{2}{c}{(109,739 docs.)} \\
        \cmidrule{2-3} \cmidrule{4-5}
        & Rec@1 & MRR100 & Rec@1 & MRR100 \\
    \midrule
    NCI & 17.91 & 28.82 & 44.70 & 56.24 \\
    BMI & 23.78 & 35.51 & 55.17 & 65.52 \\
    \hdashline[0.5pt/1pt]
    {\em ours} \\
    \makecell{GCI ($\alpha=0.25$,\\$J=4096$)}  & \textbf{32.41} & \textbf{44.36} & \textbf{56.40} & \textbf{66.20} \\
    \midrule
    \multicolumn{5}{l}{\em Ablated versions} \\
    - $\alpha=0.1$   & 31.66 & 44.12 & 55.53 & 65.74 \\
    - $\alpha=0.5$   & 29.80 & 40.93 & 53.97 & 64.13 \\
    - $\alpha=1$ & 24.43 & 36.13 & 47.87 & 58.84 \\
    - $J=1024$ & 32.16 & 44.02 & 56.19 & 65.94 \\
    - $J=300$ & 29.08 & 40.93 & 53.62 & 64.07  \\
    - No localized $\phi$ & 30.37 & 42.26 & 55.20 & 65.41 \\
    \bottomrule
\end{tabularx}
\caption{Retrieval accuracies on document indexes acquired with
different clustering methods, using a T5 retrieval model with $\sim$
30M parameters. The best scores are in bold. }
\label{tbl:gdr-mini}
\end{table}

The upper half of Table \ref{tbl:gdr-mini} compares the different
clustering-based indexing methods. On MS Marco Lite, BMI outperformed
NCI, and our GCI significantly outperformed both. Compared with BMI,
our method achieved an 8.63 (36\%) increase in Recall@1 and an 8.85
(25\%) increase in MRR@100.

Ablation tests (lower half of Table \ref{tbl:gdr-mini}) showed consistency with
the clustering evaluations reported in Table \ref{tbl:clustering}. The best
retrieval performance was achieved by $\alpha=0.25$, suggesting that it is a
robust choice for both clustering and generative retrieval. Moreover, while
reducing the sample size $J$ to smaller values (1024 and 300) decreased the
retrieval accuracy, the performance loss was insignificant for $J=1024$ compared
to $J=4096$. Furthermore, skipping the localization of $\phi$ to sub-clusters
(last row of Table \ref{tbl:clustering}) also reduced the retrieval accuracy.

\section{Discussions}
\label{sec:biasness}

\paragraph{Bias-Variance Tradeoff}
Previous research on importance sampling primarily focused on reducing
variance while maintaining an unbiased estimator (i.e., $\mathbb{E}[\hat{d}]=d$),
typically using $\alpha=1$ in Eq. \eqref{eq:ris} \citep{owen2000safe}.
However, clustering requires no such constraint of unbiasedness, as the
clustering results are invariant to scaling of $\hat{d}$ and robust to bias.

On the other hand, lowering $\alpha$ has a significant impact on reducing the
variance (or uncertainty) of the estimator $\hat{d}$, which is crucial for
$\hat{d}$ to reliably recover the true clustering structure of $d$. For language
models, the probabilities $p(\qvar|\doc)$ are typically right-skewed severely,
akin to the log-normal distribution, and reducing $\alpha$ can exponentially
reduce the variance of the estimator. This can be analyzed through the notion of
{\em effective sample size} \citep{hesterberg1995weighted}.
This explains why reducing $\alpha$ from 1 to 0.25 significantly improved
the NMI on the R2 dataset from 25.8\% to 77.8\% (Table \ref{tbl:clustering}).

However, too much reduction in $\alpha$ can lead to distortion of information,
and as the benefits of variance reduction diminish, the bias dominates the
error. As shown in Figures \ref{fig:alpha-J}(a-b), reducing $\alpha$ below 0.25
decreased the clustering performance. This reveals a tradeoff between bias and
variance in the importance sampling estimation of the KL divergence.

\paragraph{Choice of Language Model} We have focused on using the
doc2query family of models \citep{nogueira2019doc2query} in our experiments,
which are pretrained to generate queries from documents. This
choice is motivated by observing many real-world datasets
to be organized around queries or topics,
in addition to the content of the documents themselves. Nevertheless, the
doc2query models were not specifically designed for clustering, and how to
fine-tune language models for clustering is an interesting future research
direction.

\paragraph{Computational Complexity} The primary computational cost of our
method is the calculation of the probability matrix $\mathbf{P}$, which requires
evaluating $p(\query|\doc)$ for each pair of documents and queries. This cost
can be reduced by caching LLM states of $\doc$ and using low-precision
inference. We used the BF16 precision in our experiments and observed no
significant performance loss. On the R2 dataset, for example, calculation of
$\mathbf{P}$ finishes within 10 minutes on a single GPU.

\section{Conclusion}

This paper explored using generative large language models (LLMs) to
enhance document clustering. By translating documents into a broader
representation space, we captured richer information, leading to a
more effective clustering approach using the KL divergence.

Our results showed significant improvements across multiple datasets,
indicating that LLM-enhanced representations improve clustering
structure recovery. We also proposed an information-theoretic
clustering method based on importance sampling, which proved effective
and robust across all tested datasets. Additionally, our method
improved the retrieval accuracy in generative document retrieval
(GDR), demonstrating its potential for wide-ranging applications.

\section*{Acknowledgments} This work was supported by JST CREST Grant Number
JPMJCR2114. We thank the anonymous reviewers for their valuable comments and
suggestions.

\bibliography{main}

\clearpage
\appendix

\onecolumn

\begin{center}
  \LARGE\bf Technical Appendices for\\
  Information-Theoretic Generative Clustering of Documents
\end{center}
\vspace{2em}

\section{Mathematical Details}

\subsection{Convergence of the Clustering Algorithm}
\label{sec:converge}

\begin{proposition}
Algorithm \ref{alg:clustering} converges to a local minimum of the
total distortion:
\begin{equation}
D \equiv \sum_{\doc\in\docs} \hat{d}(\doc, f(\doc)),
\end{equation}
where $\hat{d}$ is defined in Eq. \eqref{eq:ris}.
\label{thm:converge}
\end{proposition}

\begin{proof}
The generative clustering algorithm \ref{alg:clustering} employs a two-step
iterative process to alternately update the cluster assignment function
$f:\docs\mapsto\{1,2,\cdots,K\}$ and the cluster centroids $\mathbf{c}_k$.
Because the total distortion $D$ decreases after each update of
the cluster assignment function $f$ and the cluster centroids $\mathbf{c}_k$,
the algorithm converges to a local minimum of $D$.
\end{proof}

However, convergence to a global optimum is not guaranteed, as this is a more
general and much challenging problem, which is not addressed in this paper. In
practice, we find that using random initialization with different seeds and
performing model selection based on total distortion, similar to the approach
used in $k$-means, improves clustering performance.

\vspace{1em}
\subsection{Optimal Cluster Centroid}
\label{sec:optim-centroid}

\begin{proposition}(Distortion-minimal cluster centroid)
For a cluster $k$, the within-cluster total distortion
\begin{align}
    D_k(\mathbf{c}_k) \equiv \sum_{\doc\in f^{-1}(k)} \hat{d}(\doc, f(\doc))
    = \sum_{\doc_i\in f^{-1}(k)} \frac{1}{J} \sum_{j=1}^J
    \mathbf{W}_{ij} \log\frac{\mathbf{P}_{ij}}{\mathbf{c}_k(j)}
    \label{eq:cluster-distortion}
\end{align}
with respect to the cluster centroid $\mathbf{c}_k$,
subject to $\sum_{j=1}^J \mathbf{c}_k(j)=1$, is minimized at:
\begin{equation}
    \mathbf{c}_k^*(j) = \frac{1}{Z}
    \sum_{\doc_i\in f^{-1}(k)} \mathbf{W}_{ij},
    ~~~~j=1,2,\cdots,J,
\end{equation}
where $Z=\sum_{j=1}^J \mathbf{c}_k^*(j)$ is the normalization term.
\label{thm:centroid}
\end{proposition}

This proposition is similar to that for Bregman hard clustering problem
\citep{banerjee2005clustering}. We provide a proof via the Lagrange multiplier
method in the following.

\begin{proof}
Rearranging Eq. \eqref{eq:cluster-distortion}, we obtain:
\begin{align}
D_k(\mathbf{c}_k) &= \frac{1}{J} \sum_{j=1}^J \sum_{\doc_i\in f^{-1}(k)}
\mathbf{W}_{ij} \log\frac{\mathbf{P}_{ij}}{\mathbf{c}_k(j)} \\
&= - \frac{1}{J} \sum_{j=1}^J \mathbf{w}_k(j) \log \mathbf{c}_k(j) + \text{constant}, \\
\end{align}
where $\mathbf{w}_k(j)=\sum_{\doc_i\in f^{-1}(k)} \mathbf{W}_{ij}$.

Using the Lagrange multiplier method to incorporate the constraint,
the Lagrangian is given by:
\begin{equation}
L(D_k, \lambda) = D_k(\mathbf{c}_k)
+ \lambda \left(\sum_{j=1}^J \mathbf{c}_k(j) - 1\right).
\end{equation}

Taking the derivative of $L$ with respect to $\mathbf{c}_k(j)$ for
every $j$, we get:
\begin{equation}
    \frac{\partial L}{\partial \mathbf{c}_k^*(j)}
    = -\frac{\mathbf{w}_k(j)}{J\mathbf{c}_k^*(j)} + \lambda = 0.
\end{equation}
This implies:
\begin{equation}
    \frac{\mathbf{w}_k(1)}{\mathbf{c}_k^*(1)} = \cdots
    = \frac{\mathbf{w}_k(J)}{\mathbf{c}_k^*(J)}
    = J \lambda.
\end{equation}

Using $\sum_{j=1}^J \mathbf{c}_k^*(j)=1$ again, we derive the optimal
solution:
\begin{equation}
    \mathbf{c}_k^*(j) = \frac{1}{Z} \mathbf{w}_k(j)
    = \frac{1}{Z} \sum_{\doc_i\in f^{-1}(k)}\mathbf{W}_{ij},
\end{equation}
where $Z=\sum_{j=1}^J \mathbf{c}_k^*(j)$ is the normalization term.
This completes the proof of Proposition \ref{sec:centroid}.
\end{proof}

\vspace{1em}
\subsection{The Second-Moment-Minimizing Proposal $\phi(\qvar)$}
\label{sec:optim-phi}

As discussed in Section \ref{sec:distortion}, we use importance
sampling to estimate $\mathbb{E}_{\qvar\sim p(\qvar|\doc)}
\log\frac{p(\qvar|\doc)}{p(\qvar|k)}$ (the KL divergence between
document $\doc$ and cluster centroid $k$) by employing an alternative
distribution $\phi$, known as the {\em proposal} distribution:
\begin{equation}
    \mathbb{E}_{\qvar\sim p(\qvar|\doc)} \log\frac{p(\qvar|\doc)}{p(\qvar|k)}
    =\mathbb{E}_{\qvar\sim\phi}
    \frac{p(\qvar|\doc)}{\phi(\qvar)}\log\frac{p(\qvar|\doc)}{p(\qvar|k)}.
\end{equation}
In importance sampling, large variation in the importance weights
$\frac{p(\qvar|\doc)}{\phi(\qvar)}$ can lead to unreliable estimates
of the KL divergence. To minimize this variance, choosing
$\phi(\qvar)=p(\qvar|\doc)$ would be ideal, as it reduces the variance
to zero. However, since $\phi$ is shared across all documents in
$\docs$, it should be selected to minimize the overall variance across
all $p(\qvar|\doc)$ distributions, rather than for any specific
document.

This intuition can be realized by minimizing the second moment
of the importance weights:
\begin{equation}
\mathbb{E}_{\phi(\qvar)}
\left[\left(\frac{p(\qvar|\dvar)}{\phi(\qvar)}\right)^{2\alpha}\right]
= \mathrm{Var}_{\phi(\qvar)}
\left(\frac{p(\qvar|\dvar)}{\phi(\qvar)}\right)^{\alpha} +
\left\{\mathbb{E}_{\phi(\qvar)}
\left[\left(\frac{p(\qvar|\dvar)}{\phi(\qvar)}\right)^{\alpha}\right]
\right\}^2,
\end{equation}

This idea is formalized in the following proposition.

\begin{proposition}(Second-moment-minimizing proposal $\phi^*$)
    The second moment of the importance weights,
    \begin{equation}
    M(\phi) \equiv \mathbb{E}_\dvar \mathbb{E}_{\phi(\qvar)}
    \left[\left(
        \frac{p(\qvar|\dvar)}{\phi(\qvar)}
    \right)^{2\alpha}\right]
    \end{equation}
    subject to the normalization condition $\sum_{\query\in\qspace}
    \phi(\query)=1$, is minimized by:
    \begin{equation}
      \phi^*(\query) = \frac{1}{Z}\left(\mathbb{E}_\dvar \left[
            p(\query|\dvar)^{2\alpha}
        \right]
        \right)^{1/2\alpha}    ~~~~~~ \forall\query\in\qspace,
        \label{eq:optim-phi-thm}
    \end{equation}
    where $Z=\sum_{\query\in\qspace} \phi^*(\query_j)$ is the
    normalization term.
    \label{thm:phi}
\end{proposition}

\begin{proof}
The constraint is incorporated using the Lagrange multipler
method, with the following Lagrangian:
\begin{equation}
L(\phi, \lambda) = M(\phi) + \lambda \left(\sum_{\query\in\qspace} \phi(\query) - 1 \right).
\end{equation}

Take the partial derivative of $L$ with respect to $\phi(\query)$ for
any $\query$. We have:
\begin{align}
\frac{\partial L}{\partial \phi(\query)}
&= \frac{\partial}{\partial \phi(\query)} \left[
    \mathbb{E}_\dvar \left(
    \sum_{\query\in\qspace}^J p(\query|\dvar)^{2\alpha} \phi(\query)^{1-2\alpha} \right) +
    \lambda \left(\sum_{\query\in\qspace} \phi(\query) - 1 \right)
    \right]
    \\
&= (1-2\alpha) \mathbb{E}_\dvar [p(\query|\dvar)^{2\alpha}]
    \phi(\query)^{-2\alpha} + \lambda.
\end{align}

Equating these derivatives to zero:
\begin{equation}
\frac{\phi^*(\query)}{\mathbb{E}_\dvar [p(\query|\dvar)^{2\alpha}]^{1/2\alpha} } = \left(
  \frac{2\alpha-1 }{\lambda}
\right)^{1/2\alpha}
\quad \forall \query\in\qspace.
\end{equation}
which are identical for all $\query\in\qspace$.
Thus,
\begin{equation}
    \phi^*(\query)\propto \left(\mathbb{E}_\dvar
    \left[ p(\query|\dvar)^{2\alpha} \right]
    \right)^{1/2\alpha}
    \quad \forall\query\in\qspace.
\end{equation}

Scaling $\phi^*$ to satisfy the constraint $\sum_{\query\in\qspace}
\phi(\query)=1$ yields the expression in Eq. \eqref{eq:optim-phi-thm},
completing the proof.

\end{proof}

Nevertheless, since $\qspace$ is infinite, $Z$ is not computable. Fortunately,
$Z$ can be factored out from $\hat{d}(\doc,k)$, meaning that setting $Z$ to any
positive value will yield the same cluster assignment and not affect the
clustering procedure. Thus, we set $Z=1$ in the paper, leading to Eq.
\eqref{eq:vmp}.

\paragraph{Localized proposal $\tilde{\phi}$.}
In a hierarchical clustering on $\docs$, the clustering algorithm
\ref{alg:clustering} is recursively applied to subclusters, denoted as
$\tilde{\docs}$, as discussed in Section \ref{sec:hierarchical}. The
second-moment-minimizing proposal is then computed using documents in
$\tilde{\docs}$ instead of $\docs$:
\begin{equation}
    \tilde{\phi}^*(\query) \gets \left(\frac{1}{|\tilde{\docs}|}
    \sum_{\doc\in\tilde{\docs}} p(\query|\doc)^{2\alpha}
    \right)^{1/2\alpha}.
\end{equation}

\vspace{5em}

\section{Additional Experiments}
\subsection{Initialization by $k$-means++}
\label{sec:kmeanspp}

A more sophisticated initialization method, $k$-means++, can be adapted to our
clustering algorithm. In the $k$-means algorithm and our Algorithm
\ref{alg:clustering}, the cluster centroids are initialized by randomly
selecting $K$ documents. The $k$-means++ algorithm \citep{arthur2006kmeanspp}
improves this initialization by selecting documents that are far apart from each
other, thus reducing the risk of converging to a suboptimal solution.

The $k$-means++ method used the following initialization procedure:
\begin{enumerate}
\item[a] Randomly select the first centroid $\mathbf{c}_1$ from
the document set $\docs$ according to the uniform distribution.
\item[b] Selecting the next cluster centroid $c_i$ as $x'\in\docs$
with probability $\dfrac{d(x')^2}{\sum_{\doc\in\docs} d(x)^2}$,
where $d(\doc)$ denote the Euclidean distance between $\doc$ and the nearest
centroid already chosen.
\item[c] Repeat Step b until $K$ centroids are selected.
\end{enumerate}

In our approach, the centroids $c_i$ are represented by the importance weights
of the selected documents, each normalized to sum to 1, as in Eq.
\eqref{eq:optim-centroid}. Moreover, because the estimated KL divergence used in
our algorithm is not necessarily positive, we modify the distance $d(\doc)$ used for
$k$-means++ initialization as follows: at iteration $i$,
\begin{align}
d(\doc) &\leftarrow \min_{k=1,\cdots,i-1} \hat{d}(\doc, k)
\quad \forall\doc\in\docs \\
d(\doc) &\leftarrow d(\doc) - \min_{\doc\in\docs} d(\doc),
\end{align}
where $\hat{d}(\doc, k)$ is the estimated KL divergence between $\doc$ and
cluster $k$ using the RIS estimator in Eq. \eqref{eq:ris}. This modification
offsets the estimated KL divergence, such that $d(x) \geq 0$, and the
probability of selecting a document $\doc$ as the next centroid to be exactly
zero if it has the smallest estimated KL divergence to the nearest centroid.

\begin{table}[htbp]
\centering
\small
\begin{tabularx}{\linewidth}{@{\extracolsep{\fill}} lrrrrrrrrrrrr}
    \toprule
    Initialization & \multicolumn{3}{c}{R2} & \multicolumn{3}{c}{R5} & \multicolumn{3}{c}{AG News} & \multicolumn{3}{c}{Yahoo! Answers} \\
    \midrule
    \multicolumn{13}{l}{\em The selected run with minimal total distortion, averaged over 100 repeated experiments} \\
    Random & 96.1 & 77.8 & 84.9 & 79.1 & 71.5 & 69.1 & 85.9 & 64.2 & 67.2 & 60.7 & 43.8 & 35.5 \\
    $k$-means++ & 96.1 & 77.8 & 84.9 & 79.8 & 72.0 & 69.6 & 85.9 & 64.3 & 67.3 & 60.5 & 43.7 & 35.5 \\
    \midrule
    \multicolumn{13}{l}{\em Single run, averaged over 100 repeated experiments} \\
    Random & 89.0 & 65.9 & 69.4 & 72.7 & 67.4 & 63.5 & 77.0 & 59.9 & 58.5 & 55.8 & 42.4 & 33.4 \\
    $k$-means++ & 89.6 & 66.1 & 69.8 & 72.4 & 67.5 & 63.3 &  77.4 & 60.2 & 59.3 & 55.5 & 42.2 & 33.2 \\
    \bottomrule
\end{tabularx}
\caption{
Results of the $k$-means++ initialization compared with random initialization on
four datasets.
}
\label{tbl:kmeanspp}
\end{table}

We tested the $k$-means++ initialization on the datasets in Table
\ref{tbl:dataset}, comparing it with the default random initialization. The
results are shown in Table \ref{tbl:kmeanspp}. The upper half of the table was
acquired with the same setting as in the main text (Table \ref{tbl:clustering}),
i.e., with a selecting procedure on ten different seeds. The lower half of the
table shows the results without the selecting procedure. All results are
averaged over 100 repeated experiments.

The $k$-means++ initialization did not significantly improve the clustering
performance compared with random initialization, either with or without the
selecting procedure. This could be because the datasets used in our experiments
have relatively smaller numbers of clusters, making the impact of initialization
less significant.

\vspace{1em}
\subsection{Clustering with Misspecified Number of Clusters}
\label{sec:misspec-K}

\begin{table}[h]
\centering
\small
\begin{tabular}{lrrrrrr}
    \toprule
    $K$ & 2 & 4 & 5\textsuperscript{*} & 6 & 10 & 20 \\
    \midrule
    SBERT (All-minilm-l12-v2) & 31.9 & 53.7 & 58.6 & 50.8 & 37.9 & 21.9  \\
    GC (ours) & 53.9 & 64.5 & 69.1 & 59.8 & 38.6 & 22.0 \\
    \bottomrule
\end{tabular}
\caption{Clustering performance on the R5 dataset with misspecified number of
clusters, measured by Adjusted Rand Index (ARI). The ground-truth number of
clusters is 5, as marked with an asterisk.}
\label{tbl:misspec-K}
\end{table}

Table \ref{tbl:misspec-K} shows the clustering performance on the R5 dataset
when the number of clusters $K$ is misspecified. The ground-truth number of
clusters is 5. Our GC (last row) is compared with the best-performing
SBERT-based method on this dataset (first row).

When the number of clusters is misspecified, the clustering performance of the
two methods degraded. However, our method still outperformed the SBERT-based
method in all cases. This suggests that our method is robust to misspecified $K$
values and is effective in real-world applications where the number of clusters
is unknown.

\end{document}